\documentclass{article}

     \usepackage[preprint]{neurips_2019}

\usepackage[utf8]{inputenc} 
\usepackage[T1]{fontenc}    
\usepackage{hyperref}       
\usepackage{url}            
\usepackage{booktabs}       
\usepackage{amsfonts}       
\usepackage{nicefrac}       
\usepackage{microtype}      
\usepackage{color}

\usepackage{caption}
\usepackage{subcaption}
\usepackage{float}
\usepackage{amsmath} 
\usepackage{amsthm}

\newtheorem{theorem}{Theorem}

\newtheorem{corollary}{Corollary}

\theoremstyle{remark}

\newtheorem*{remark}{Remark}

\newcommand{\HH}{\mathrm {H}}

\newcommand{\No}{\text{No}}

\newcommand{\E}{\mathbb{E}}

\newcommand{\Cov}{\mathrm{Cov}}
\newcommand{\be}{\begin{equation*}
  \begin{aligned}}
    
\newcommand{\ee}{ \end{aligned}
\end{equation*}
}      

\newcommand{\bel}{\begin{equation}
  \begin{aligned}}
    
\newcommand{\eel}{ \end{aligned}
\end{equation}
}

\usepackage{graphicx}
\graphicspath{ {./images/} }

\title{Tuning-Free Disentanglement\\ via Projection}

\author{%
   Yue Bai 
   \\
  Department of Statistics \\
  University of Florida\\
  Gainesville, FL  32603\\
  \texttt{baiyue@ufl.edu} \\
  \And
  Leo L. Duan \\
  Department of Statistics\\ 
  University of Florida\\
  Gainesville, FL  32603\\
  \texttt{li.duan@ufl.edu} \\
}

\begin{document}

\maketitle

\begin{abstract}
In representation learning and non-linear dimension reduction, there is a huge interest to learn the ``disentangled'' latent variables, where each sub-coordinate almost uniquely controls a facet of the observed data.  While many regularization approaches have been proposed on  variational autoencoders, heuristic tuning is required to balance between disentanglement and loss in reconstruction accuracy --- due to the unsupervised nature, there is no principled way to find an optimal weight for regularization. Motivated to completely bypass regularization, we consider a projection strategy: modifying the canonical Gaussian encoder, we add a layer of scaling and rotation to the Gaussian mean, such that the marginal correlations among latent sub-coordinates become exactly zero. 
This achieves a theoretically maximal disentanglement, as guaranteed by zero cross-correlation between one latent sub-coordinate and the observed varying with the rest. Unlike regularizations, the extra projection layer does not impact the flexibility of the previous encoder layers, leading to almost no loss in expressiveness. This approach is simple to implement in practice. Our numerical experiments demonstrate very good performance, with no tuning required.
 \end{abstract}

\section{Introduction}

Unsupervised dimension reduction aims to learn a latent low-dimension representation \(z_i \in \mathbb{R}^d\) from the observed data \(x_i \in \mathbb{R}^p\) with \(d\ll p\). In a generative model framework, one can imagine that the observed and latent are linked via a likelihood $p[x_i \mid f^*(z_i)]$ along with a prior $p(z_i)$; to simplify estimating the posterior distribution of $z_i$, \citet{kingma2013auto} proposes a variational distribution \(q[z_i\mid f(x_i)]\); by letting \(f\) and \(f^*\) be neural networks, this has an interesting connection to the classical autoencoder, and is named as the variational autoencoder (VAE). The canonical VAE loss function is
\[- \mathbb{E}_{q[z_i \mid f(x_i)]} \log p[ x_i \mid f^*(z_i)] + KL\{q[z_i \mid f(x_i)]|| p(z_i)\}\]
where \(KL\) denotes the Kullback-Leibler (KL) divergence. 
VAE has generated a huge interest as it was discovered to have a ``disentangling'' effect --- in particular, when the prior $p(z_i)$ follows an independent form, such as independent multivariate Gaussian, in the variational posterior, each sub-coordinate  $z_{i,1},\ldots, z_{i,d}$ can sometime correspond to a unique facet of the $x_i$. This provides appealing real-world interpretation.

Intuitively, this effect is due to the regularization of the KL divergence term. To pull the posterior closer to the independent structure, it is natural to modify this term.
There is an active literature. Among others, \(\beta\)-VAE \citep{higgins2017beta} increases the weight of the KL divergence to \(\beta>1\), with \(p(z_i)=\mathcal{N}(z_i ; 0,I)\); 
FactorVAE \citep{kim2018disentangling} considers total correlation regularization on the marginal \(q(z_i)\) towards a factorized form \(\prod_{k=1}^d q(z_{i,k} )\);   readers can find alternative approaches such as InfoGAN \citep{chen2016infogan} and Deep InfoMax \citep{hjelm2018learning}, although we will focus on reviewing VAEs here for concise exposition.

It is found that in order to have perceivable disentanglement, the weight of regularizer needs to be ``sufficiently large''. This is understandable, because of the high complexity of decoder $f^*$, it is hard to guarantee that the near-independence of $z$ can pass to $f^*(z)$. On the other hand, too much regularization can lead to deteriorated performance and large reconstruction error between $f^*(z_i)$ and $x_i$.  Arguably, the primary reason is that the regularization target is a deteriorated representation of the data, compared to the original $q[z_i \mid f(x_i)]$ obtained from end-to-end training. For example, standard multivariate Gaussian prior  \(p(z_i)\) KL divergence has no information from \(x_i\) at all; independent $q(z_i)$ in total correlation penalty is intractable and requires approximation. 
\cite{achille2018emergence} provides a careful quantification on the disentanglement-accuracy tradeoff in information theory; on the other hand, there is a lack of statistical theory and tuning is mostly heuristic.

Conceptually related to our projection approach, there are regularizations applied on the Gaussian covariance among the latent sub-coordinates \citep{cheung2014discovering,cogswell2015reducing,kumar2017variational}; roughly speaking, they penalize the deviation of correlations/covariances among $z_{.,1},\ldots, z_{.,d}$ from zero, as a tractable surrogate target for independence (for the later, mutual information is known to have sensitivity issue). However, there are two critical issues in those approaches: 1. the correlation regularization is applied on the {\em whole} encoder $f$ --- forcing a nonlinear function to satisfy a set of linear constraints, it negatively impacts its expressiveness; 2. there is a numerical error issue in penalizing non-zero sample covariance, which as an estimator converges approximately $O(n^{-1/2})$ \citep{gotze2010rate} --- for example, two independent Gaussian $z_{i,1}$ and $z_{i,2}$ with $n=10^6$ will still show a non-zero correlation around $10^{-3}$; this means larger penalty is not meaningful, even though it incurs more reconstruction error.

Realizing the problems with regularization, we propose a fundamentally different strategy. Our key idea is to induce disentanglement directly in $q[z_i\mid f(x_i)]$. Motivated by Stein's lemma \citep{stein1981estimation}, we put the latent variable in a constrained space with {\em exactly} zero marginal correlations, this leads to no cross-correlation  between one $z_{i,k}$ and the decoded image that varies with other $z_{i,k'}$'s. Practically, this representation can be easily obtained with an extra projection layer of scaling and rotation on the Gaussian encoder mean, and letting it go through end-to-end training. Intuitively, the last-layer projection does not impact the flexibility of previous layers, retaining the expressiveness of the encoder.

\section{Method}

\subsection{Background}
The standard VAE \citep{kingma2013auto} has the following variation distribution (represented by $q [z_i \mid f(x_i)] $ in the introduction):
\begin{equation}\label{eq:vanilla_vae}
        z_i\mid g(x_i),h(x_i) \sim \No [ g(x_i), h(x_i) ]
\end{equation}
where $g:  \mathbb{R}^p \to  \mathbb{R}^d$ outputs a mean vector and $h:  \mathbb{R}^p \to \mathbb{R}^{d}$ outputs a diagonal covariance;  we will use $g_k$ to denote the $k$th element in the output of $g$.

Even though $q[z_i \mid f(x_i)]$ has conditionally independent form due to the diagonal covariance, the marginal $q(z_i)$ does not --- based on the law of total covariance, the marginal covariance for $k\neq k'$
\be
 \Cov (z_{i,k}, z_{i,k'}) =&
\mathbb{E}_X \Cov_Z(z_{i,k}, z_{i,k'} \mid x_i) 
+\Cov_X \big[ \mathbb{E}_Z(z_{i,k} \mid x_i),  \mathbb{E}(z_{i,k'} \mid x_i) \big ]
\\ =& \Cov_X [ g_k(x_i), g_{k'}(x_i)],
 \ee
 where we use capital $Z$ and $X$ to represent the random variables, based on the variational and empirical distributions, respectively. As quantified later, non-zero latent covariance can lead to entanglement.

\subsection{Projection-VAE}
We consider an alternative mapping for the Gaussian mean $g^*$ in a constrained space
\be
\mathcal{G}^*=\{ g^*: \Cov_X [ g^*_k(x_i), g^*_{k'}(x_i)]=0, \forall k\neq k'\}.
\ee
To obtain $g^*$ as a projection of $g$ into $\mathcal{G}$,
we minimize the distance to the original $g$ in their means $\|\mathbb{E}_Xg^*(x_i)-\mathbb{E}_Xg(x_i)\|$, this leads to a simple closed form solution
\be
g^*(x_i)= L'  [g(x_i)-  \mu_g ]+ \mu_g
\ee
where $\mu_{g} = \mathbb{E}_{X}g(X)=n^{-1}\sum_i g(x_i), \Sigma_g = \mathbb{E}_{X} [g(X) - \mu_{g}] [g(X) - \mu_{g}]'=n^{-1}\sum_i [g(x_i)- \mu_g][ g(x_i)- \mu_g]'$. And the $d\times d$ matrix $L$ is from the eigendecomposition
\be
\Sigma_g=UDU',\qquad L=UD_*^{1/2}
\ee
where $D_* = (D+I \epsilon)^{-1}$ is a diagonal matrix, with $\epsilon> 0 $ small in case $\Sigma_g$ is rank deficient (otherwise we set $\epsilon=0$).

Multiplying $U$ and $D_*^{1/2}$ correspond to rotation and scaling. And we can verify that marginal
\be
 \Cov [ z_{i,k},  z_{i,k'}]=0
 \ee
 for $k\neq k'$, and diagonal variance $(D_*^{1/2}DD_*^{1/2})_{k,k} + \mathbb{E}_X h_k(X)$.

To implement this projection, it only requires one extra linear transform layer in the decoder (Figure~\ref{fig:arch}). Denoting this layer by $m$ (hence $g^*(x)=m[g(x)]$), we have a projected variational distribution
\begin{equation}\label{eq:proj}
        z_i\mid m[g(x_i)], h(x_i) \sim \No \big \{  m[g(x_i)],    h(x_i)  \big \}, \quad  m[g(x_i)]= L^{-1}  [g(x_i)-  \mu_g ]+ \mu_g.
\end{equation}
We refer to this method as Proj-VAE. Notice the loss function is similar to the standard VAE
\[- \mathbb{E}_{q\{ z_i \mid m[g(x_i)], h(x_i)    \} } \log p[ x_i \mid f^*(z_i)] + KL\{q[z_i \mid m[g(x_i)], h(x_i)]|| p(z_i)\}.\]
For simplicity, we keep prior $p(z_i)$ the same as $\No(0,I)$, as in the canonical VAE.

Since the eigen-decomposition, rotation and scaling are all differentiable, the computation can be easily carried out using the same end-to-end training with re-parameterization trick \citep{kingma2013auto}.

\begin{figure}[H]
     \centering
         \centering
         \includegraphics[width=.7\linewidth, height=1.7in]{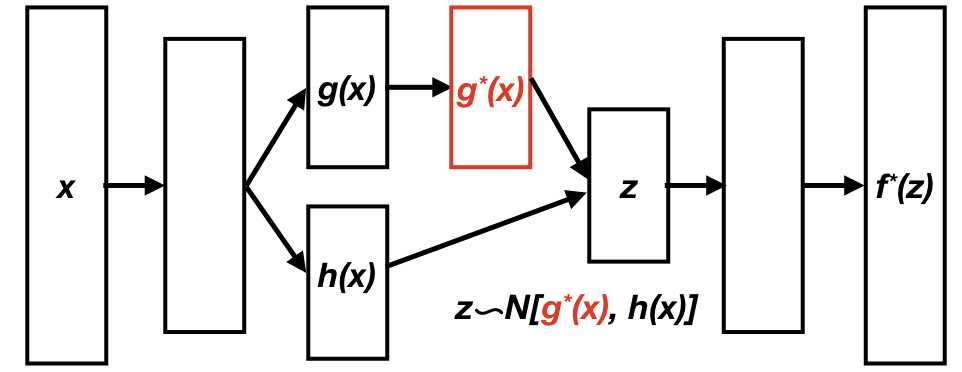}
         \caption{Illustration of network architecture of projection-VAE. It only takes a simple linear projection (red) on the canonical VAE, but ensures the variational distribution has marginal $\Cov (z_{i,k}, z_{i,k'})=0$ exactly for $k\neq k'$, benefiting disentanglement in $f^*(z)$ via Stein's covariance identity. \label{fig:arch}}
\end{figure}

\subsection{Guarantee of Maximal Disentanglement}
We now justify why exactly zero correlation is vital for maximal amount of disentanglement.

In order to do so, we need to first formalize the idea of ``disentanglement'', by rephrasing maximal disentanglement as ``zero entanglement'' --- specifically, one latent sub-coordinate {\em uniquely} controlling one observed facet, is equivalent to the other observed facets {\em not depending} on this sub-coordinate. Denoting latent variable as a random vector $\vec Z = (Z_1,\ldots, Z_d)$, we consider another vector by fixing the $k$th sub-coordinate to constant $c_k$
\be
& \vec Z^{\setminus k} := 
(Z_1,\ldots, Z_{k-1}, c_k, Z_{k+1},\ldots, Z_d),
\ee
we can now quantify the dependence via the cross-covariance between $Z_k$ and $f^*(\vec Z^{\setminus k})$.
\bel\label{eq:entanglement_measure}
\text{Cross-Cov}[  Z_k,  f^*( \vec Z^{\setminus k} ) ] = \mathbb{E}[Z_k -   \mathbb{E}Z_k][f^*( \vec Z^{\setminus k} ) -  \mathbb{E} f^*( \vec Z^{\setminus k} )].
\eel
Therefore, the closer \eqref{eq:entanglement_measure} is to zero, the larger the disentanglement is.

Despite the complexity of the decoder $f^*$, we can analytically compute the cross-covariance using  Stein's covariance identity \citep{stein1981estimation}. We first start with the Gaussian assumption, and quantifiy the mathematical source of entanglement; then we show removing latent correlation leads to zero entanglement, which can be extended to a broad class of continuous distributions.

\begin{theorem}[Entanglement under Gaussian encoder]
Let \(\vec{Z}=(Z_1,\dots,Z_d)'\) be a Gaussian random vector. Suppose \(f^*\) is a neural network  such that \(\partial f^* / \partial z_k\) is continuous almost everywhere and \(\E\vert  \partial  f_q^*(\vec Z)  /\partial z_k \vert<\infty\) for $q=1,\ldots, p$, then
\bel \label{eq:breakdown_entanglement}
\textup{Cross-Cov}[ Z_{k},  f^*_q( \vec{Z}^{\setminus k}) ] = \sum_{l=1: l\neq k}^d \Cov(Z_{k},Z_{l})\E[\frac{\partial f^*_q(\vec{Z}^{\setminus k} )}{\partial Z_{l}}].
\eel
\end{theorem}
\begin{proof}
The proof is a minor extension based on Lemma 1 from \cite{liu1994siegel}, consider $\vec Z = \mu+ L\vec Y$ with $\vec Y\sim \No[0,I]$, $LL'=\Sigma_z$. Using Stein's lemma,
\be
\Cov[ Y_l,  g_1( \vec{Y}) ] &= \mathbb{E}[\frac{\partial g_1(\vec Y )}{\partial y_{l}}].
\ee
for $l=1,\ldots,d$ and almost everywhere differentiable function $g_1:\mathbb{R}^p\to \mathbb{R}$ such that the right hand side is finite. Taking $g_1(\vec Y)= g_2[( \mu +L \vec Y )]$, we have the $d\times 1$ covariance matrix $\Cov[ \vec Z, g_2(\vec Z)] = \Sigma_z \nabla g_2(\vec Z)$. Using $g_2(\vec Z)=f^*_q(\vec Z^{\setminus k})$, we have $\partial g_2 / \partial z_k=0$; calculating the matrix product yields the result.
  \end{proof}
  \begin{remark}
        the linear form of \eqref{eq:breakdown_entanglement}  shows an intuitive decomposition of entanglement:  each latent $Z_l$ has some influence on the decoder output (quantified as the partial derivative), and the latent covariance acts as the weight coefficient to pass the influence from $Z_l$ and intertwine with $Z_k$. Clearly, zero covariances among $Z_k$'s completely remove the cross-covariance.
\end{remark}

\begin{corollary}
When the condition in Theorem 1 is satisfied and $\textup{Cov}(Z_k, Z_{l})=0$ for any $k\neq l$, then $\textup{Cross-Cov}[Z_k, f^*_q( \vec{Z}^{\setminus k}) ] = 0$ for all $k=1,\ldots, d$ and $q=1,\ldots, p$.
\end{corollary}
Thus far, we use the Gaussian assumption on $\vec Z$ to illustrate the source of entanglement. One may argue the marginal of $\vec Z$ could deviate from Gaussian-ness --- indeed, when including the randomness of its mean, we can expect potential longer tail than Gaussian. Therefore, we now relax this to general elliptical distribution --- a large distribution class  including Laplace, $t$-distribution, etc.
  \begin{theorem}[Maximal disentanglement under uncorrelated elliptical encoder]
Let \(\vec{Z}=(Z_1,\dots,Z_d)'\) be a continuous random vector from an elliptical distribution with density $h(Z)\propto \mathcal{K}(\vec Z' \Omega \vec Z)$ with $\Omega$ positive definite matrix, mean $ \mathbb{E}Z_k= \mu_k$ and variance $ \Cov(Z_k)=\sigma^2_k<\infty$ for each element. Suppose \(f^*\) is a neural network  such that \(\partial f^* / \partial z_k\) is continuous almost everywhere.
   If \(\Cov(Z_k,Z_{k'})=0\) for \(k \neq k'\), then $\textup{Cross-Cov}[Z_k, f^*_q( \vec{Z}^{\setminus k}) ] = 0$ for all $k=1,\ldots, d$ and $q=1,\ldots, p$.
\end{theorem}
\begin{proof}
  Consider univariate $Y_k$ with density $h$, mean $0$ and variance $1$,   \cite{goldstein1997stein}  showed that there always exists another random variable  $W_k$ with density
   \[
	 h(w_k) = \int_{w_k}^{\infty} t h(t)dt,
    \]
which is often referred to as the ``$Y_k$-zero biased'' distribution. It has the property:
  \[
	\mathbb{E}_{Y_k} Y_k g_1(Y_k) = 
	 \mathbb{E}_{W_k}[ \frac{\partial  g_1(W_k)}{\partial w_k}],
\]
which holds for any almost everywhere differentiable $g_1:\mathbb{R}\to \mathbb{R}$, such that the right hand side is finite. Naturally, this holds for $g_1(Y_k) = g_1(Y_k; \vec y ^{\setminus k})$ with any parameter $\vec y ^{\setminus k}$. Consider
\bel \label{eq:proof2}
g_1(Y_k; \vec y ^{\setminus k}) = 
 f^*_q(\sigma_1 y_1+\mu_1,\ldots,c_k, \ldots, \sigma_d y_d +\mu_d)
 \eel
  hence ${\partial  g_1(W_k)}/{\partial w_k}=0$ almost everywhere. Denoting $m_{k\mid (k)}=\mathbb{E}(Y_k \mid \vec Y^{\setminus k})$,
 \be
 \mathbb{E}_{\vec Y} Y_k 
 g_1(Y_k; \vec Y\setminus Y_k)& =  
\int \int [y_k -m_{k\mid (k)}+ m_{k\mid (k)}]  g_1(y_k; \vec y^{\setminus k}) h(y_k \mid \vec y^{\setminus k} ) h(\vec y^{\setminus k} ) \textup{d} y_k \textup{d} \vec y^{\setminus k} \\
& = 
\int  \mathbb{E}_{W_k\mid \vec Y^{\setminus k}}[ \frac{\partial  g_1(W_k; \vec y^{\setminus k})}{\partial w_k}] h(\vec y^{\setminus k} ) \textup{d} y_k \textup{d} \vec y^{\setminus k} \\
&+
\int \int m_{k\mid (k)} g_1(y_k; \vec y^{\setminus k}) h(y_k \mid \vec y^{\setminus k} ) h(\vec y^{\setminus k} ) \textup{d} y_k \textup{d} \vec y^{\setminus k} .\\
 \ee
 The first term on the right hand side is zero as proved above. The second term is also zero due to $m_{k\mid(k)}=m_k=0$, as the property of elliptical distribution \citep{owen1983class}.
  
Notice the above left hand side is equal to 
 \[
  \mathbb{E}_{\vec Y} [Y_k
 g_1(Y_k; \vec Y\setminus Y_k)] -   (\mathbb{E}_{Y_k} Y_k) [\mathbb{E}_{\vec Y}   g_1(Y_k; \vec Y\setminus Y_k) ]
 =\Cov[Y_k, g_1(Y_k; \vec Y\setminus Y_k)]
 \]
  taking $Z_k= \sigma_kY_k+\mu_k$ and \eqref{eq:proof2}, we have $\Cov[Z_k, f^*_q( \vec{Z}^{\setminus k})]=0$.
\end{proof}  
  \begin{remark}
  As we have generalized beyond Gaussian class, this means we do not need independence among $\vec Z$ to induce disentanglement.
    \end{remark}

\subsection{Preserving Information}

Intuitively, rotation and scaling do not change the amount of information stored in the encoder mean $\{ g(x_1),\ldots, g(x_n)\}$.  One can quantify the `expressiveness' via entropy, i.e. the information diversity. We now show this is preserved under bijective transform.
\begin{theorem}[Invariance of Entropy under Bijective Function]
For discrete random variable \(X\) and an bijective function \(m\),
\[                \HH\big \{ m[g(X)]\big \} = \HH\big[ g(X)\big] .\]
\end{theorem}
\begin{proof}
  Consider the joint entropy, decomposed as the sum of marginal and conditional entropies 
 $$\HH \{m[g(X)],g(X) \}=\mathrm {H} [g(X)]+\mathrm {H}\{ m[g(X)]|g(X)\}=\mathrm {H} \{  m[g(X)]\}+ \mathrm {H} \{ g(X)|m[g(X)]\}.$$

 Since \(m\) is a deterministic function,
$\mathrm {H}\{ m[g(X)]|g(X)\}=0$ and 
  thus 
\(\mathrm {H} [g(X)]=\mathrm {H} \{  m[g(X)]\}+ \mathrm {H} \{ g(X)|m[g(X)]\},\)
so \(\mathrm {H} [g(X)] \geq \HH\big \{ m[g(X)]\big \} .\)

Since \(m\) is bijective, then there exits an inverse function \(m^{-1}\) such that \(g(X)=m^{-1}\{m[g(X)]\}\), then applying the previous result yields
\(\mathrm {H} [g(X)] \leq \HH\big \{ m[g(X)]\big \} \). Combining two pieces yields the result.
\end{proof}

\begin{remark}
In our case, \( m[g(x_i)] = L' [g(x_i)-  \mu_g ]+ \mu_g\), where \(L\) is a \(d \times d\) invertible matrix; hence this is an bijective function.
\end{remark}

Since zero correlation constraint is directly accommodated when projecting $g$ into $g^*$, all the previous layers in $g$ remain fully flexible. This is fundamentally different from correlation-based regularization that applies on the whole encoder; the practical difference in performance will also be illustrated in the data experiments.

\section{Data Experiments}

We now use  common benchmark datasets to demonstrate good performance of Proj-VAE --- to clarify, we do {\em not} claim always superior performance than competitors --- as the evaluation of disentanglement can be quite subjective. Rather, we want to emphasize on the tuning-free property of Proj-VAE, with good reconstruction accuracy. This makes it an appealing tool for almost automatic dimension reduction.


\subsection{Reconstruction Accuracy}

To illustrate that projection has  almost no impact on expressiveness, we use the images of Frey Face. Since this is a small dataset with $n=1,900$, it is easy to train most models to their minimal loss --- this is important for a fair comparison on the reconstruction accuracy. For simplicity, we run four models with $d=3$ latent dimension, and only one hidden layer in encoder and decoder with $256$ units.

Figure~\ref{fig:freyface} shows that the reconstructions from Proj-VAE are almost indistinguishable to the canonical VAE. On the other hand, the correlation-based regularization  with penalty $\gamma \|\Sigma_g -I_d\|_1$ and $\gamma=100$,  severely impacts the reconstruction. This demonstrates the sharp distinction between projection and regularization.

\begin{figure}[h]
     \centering
         \includegraphics[width=.7\linewidth]{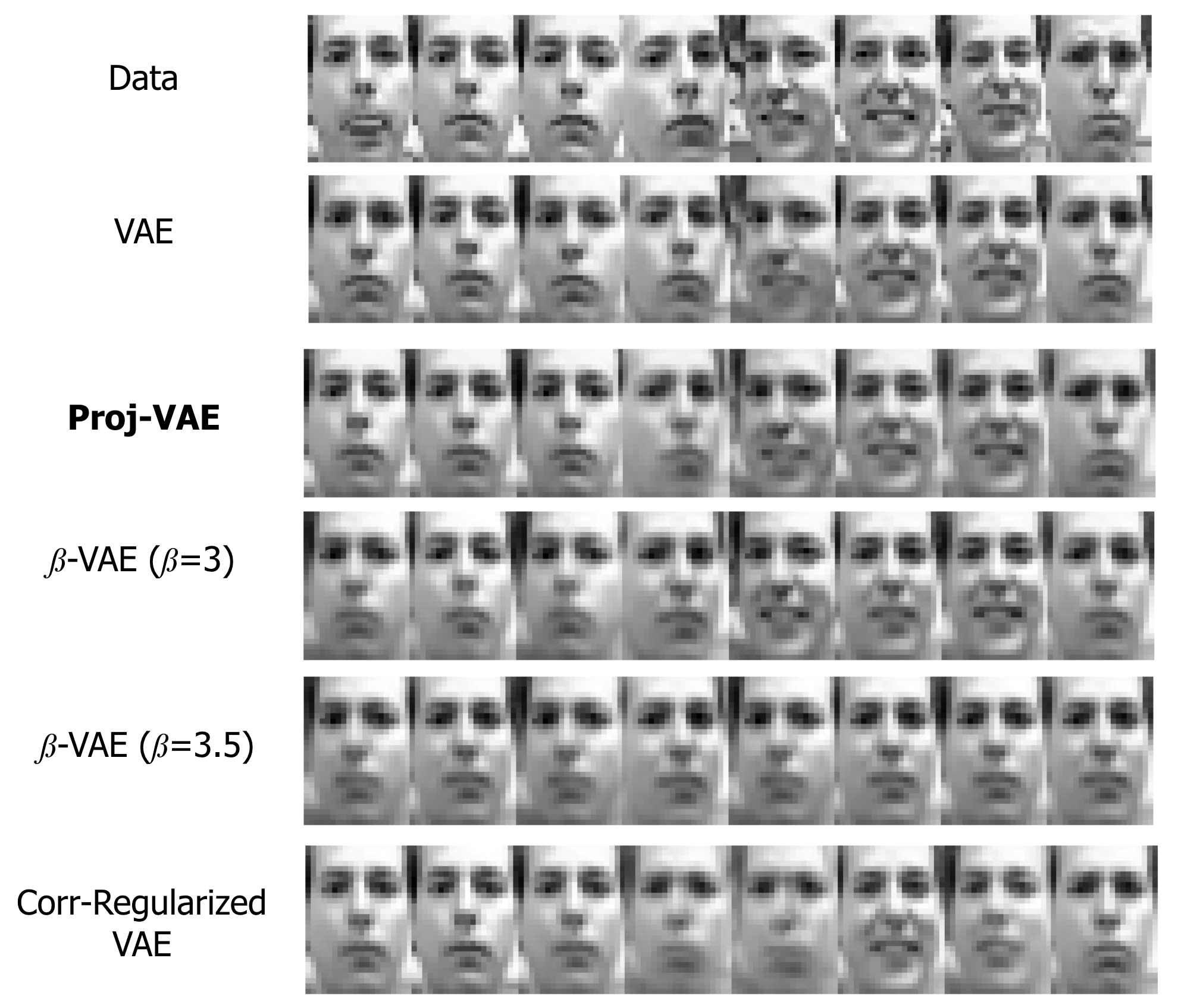}
            \caption{Comparison of the reconstructed images to the original data (first row),
             using canonical VAE, Proj-VAE, $\beta$-VAE and Correlation-regularized VAE. Proj-VAE shows almost identical reconstruction to the canonical VAE (with average binary cross-entropy errors $346.2$ and $344.0$, respectively);  correlation-based regularization (last row) severely impacts the expressiveness (with error $350.1$); $\beta$-VAE requires a careful tuning as slightly increased $\beta$ can result in loss of feature (mouth).
             \label{fig:freyface}
             }
\end{figure}

To show the effectiveness of removing correlation via projection, Figure~\ref{fig:decor} plots the correlation matrices $\text{Corr} [ g^*_k(x_i), g^*_{k'}(x_i)]$ on the scale of $\log_{10}$ of absolute value. The results are compared against the $\text{Corr} [ g_k(x_i), g_{k'}(x_i)]$ in correlation-regularized VAE. The projection achieves precision around $10^{-7}$, which is the float precision in typical GPU-based program; whereas penalty with $\gamma=100$ on the correlation can only yield $10^{-2}$. Further increasing $\gamma$ in regularization is not practical as it incurs more reconstruction error.

\begin{figure}[h]
     \centering
     \begin{subfigure}[b]{.48\linewidth}
         \centering
         \includegraphics[width=.6\linewidth]{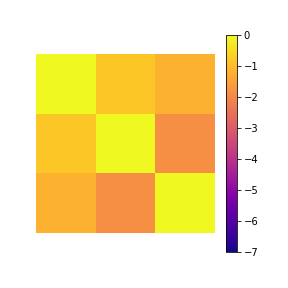}
         \caption{Correlation-based Regularization}
     \end{subfigure}
         \begin{subfigure}[b]{.48\linewidth}
         \centering
         \includegraphics[width=.6\linewidth]{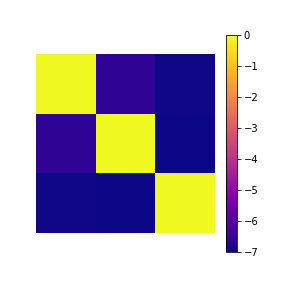}
         \caption{Projection}
     \end{subfigure}
             \caption{ Comparing the numeric values of correlation matrices ( on the scale of $\log_{10}|\text{Corr|}$) in the Gaussian means for latent $z$.
               \label{fig:decor}
}
\end{figure}

\subsection{Disentanglement Performance}

To demonstrate the disentanglement performance in a relatively large dataset, we use the chairs dataset  \citep{aubry2014seeing}, which contains $n=86,366$ images of chair CAD models. We use $d=16$ and the same architecture of convolutional network as described in \cite{higgins2017beta}.

We select three factors and present them in Figure~\ref{fig:chairs}. In Proj-VAE, the azimuth rotation and backrest height of the chairs are clearly separated (panel a). In comparison, the canonical VAE (panel b) and  $\beta$-VAE (panel c, $\beta=5$) have those two factors mixed together, although we expect $\beta$-VAE can disentangle them with larger $\beta$.
\begin{figure}[h]
     \centering
 \begin{minipage}[t]{.5\textwidth}
     \begin{subfigure}[b]{1\linewidth}
         \centering
         \includegraphics[width=2.6in,height=0.7in]{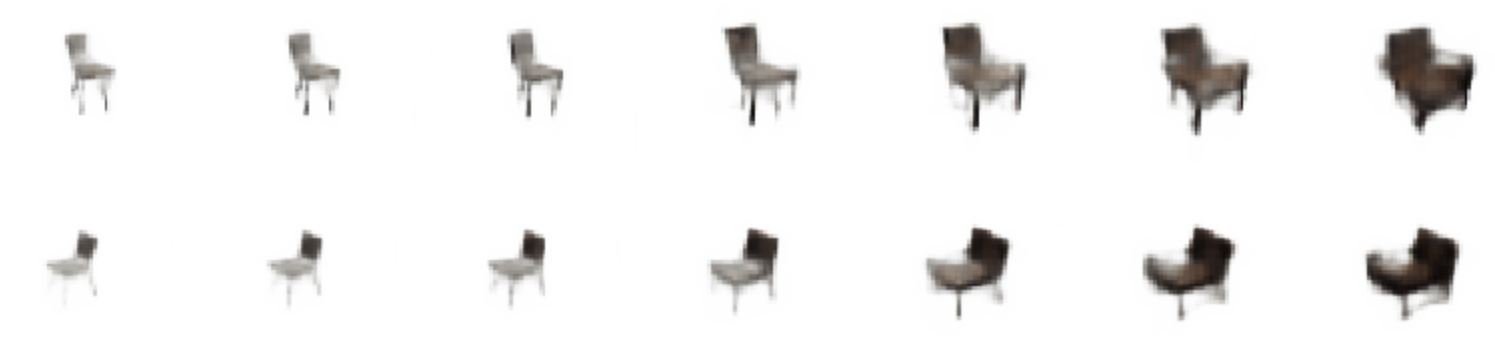}
         \caption*{i. Scale}
     \end{subfigure}
         \begin{subfigure}[b]{1\linewidth}
         \centering
         \includegraphics[width=2.6in,height=0.7in]{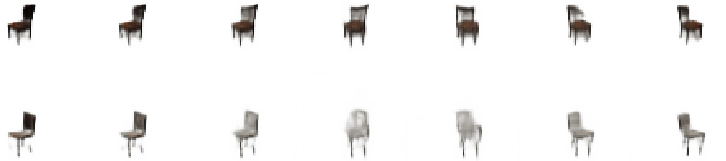}
         \caption*{ii. Rotation}
     \end{subfigure}
               \begin{subfigure}[b]{1\linewidth}
         \centering
         \includegraphics[width=2.6in,height=0.7in]{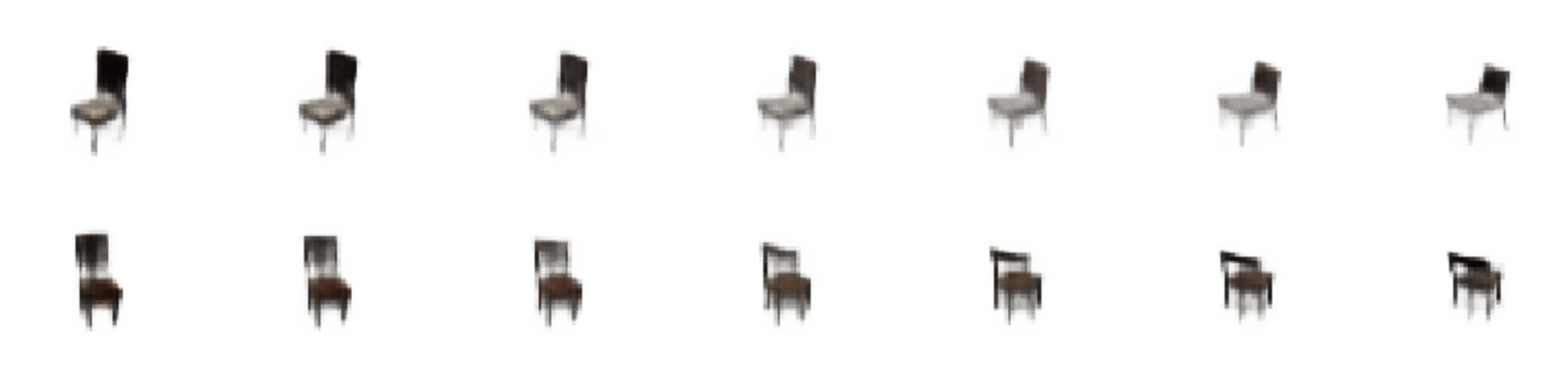}
         \caption*{iii. Backrest height}
     \end{subfigure}
    \subcaption{Disentangled representation learned by Proj-VAE.}
\end{minipage}
\begin{minipage}[t]{.49\textwidth}
         \begin{subfigure}[b]{1\linewidth}
         \centering
         \includegraphics[width=2.6in,height=0.7in]{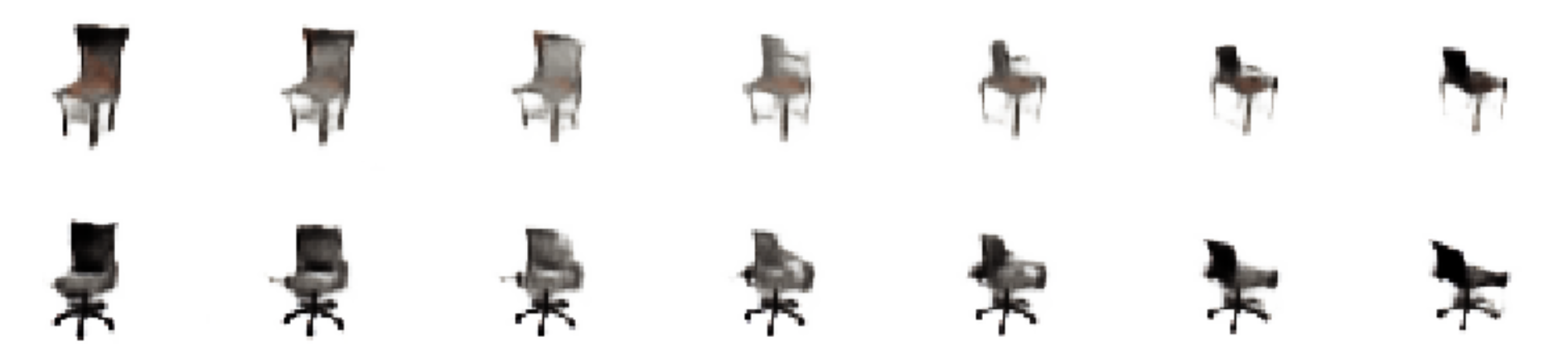}
         \caption*{Rotation and Backrest height}
     \end{subfigure}
                   \subcaption{Two factors are entangled in the canonical VAE.}
                    \vskip 0.4in
              \begin{subfigure}[b]{1\linewidth}
         \centering
         \includegraphics[width=2.6in,height=0.7in]{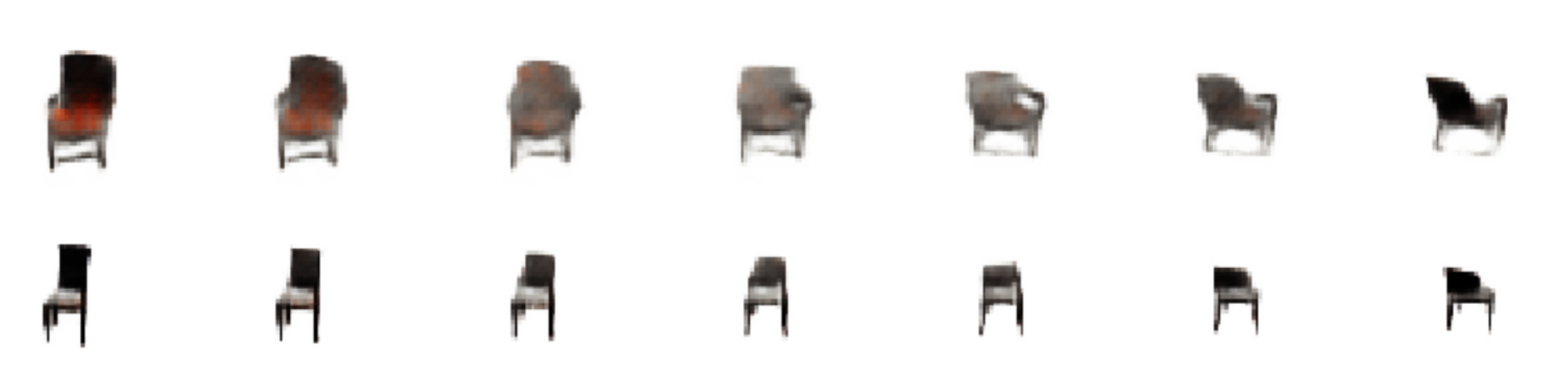}
         \caption*{Rotation and Backrest height}
                       \subcaption{Two factors are still entangled in the $\beta$-VAE when $\beta=5$.}
     \end{subfigure}
\end{minipage}
             \caption{  Generated chair images using different VAEs: each row of the data are simulated with one $z_{.,k}$ traversing
 from $-3$ to $3$ (the approximate learned range), while keeping the other factors fixed. Proj-VAE learns disentangled representation of rotation and backrest height, compared to canonical VAE (panel b) and $\beta$-VAE with $\beta=5$.
        \label{fig:chairs}}
\end{figure}

We also apply Proj-VAE on a large data CelebA \citep{liu2015deep}, which includes $n=202,599$ celebrity images over 3 color channels. We choose $d=32$ as selected by $\beta$-VAE \citep{higgins2017beta} and FactorVAE \citep{kim2018disentangling}. Our approach learns all the factors previously reported, such as skin color, image saturation, age, gender, rotation, emotion and fringe (discovered in $\beta$-VAE with $\beta=250$ and FactorVAE with $\gamma=6.4$). Interestingly, It also discovers several new factors, such as the presence of sunglasses, background brightness and color tint. Again, existing approaches can find the same factors under appropriate tuning --- for example, sunglasses factor can be recovered in $\beta$-VAE with much smaller $\beta=10$ (this seems coherent with the fact that those variations are on a finer scale); our advantage is that there is no tuning required, hence it takes only one time of training.

\begin{figure}[h]
      \centering
      \begin{subfigure}[b]{0.329\linewidth}
          \centering
          \includegraphics[width=\linewidth,trim={6cm 3cm 5cm 3cm},clip]{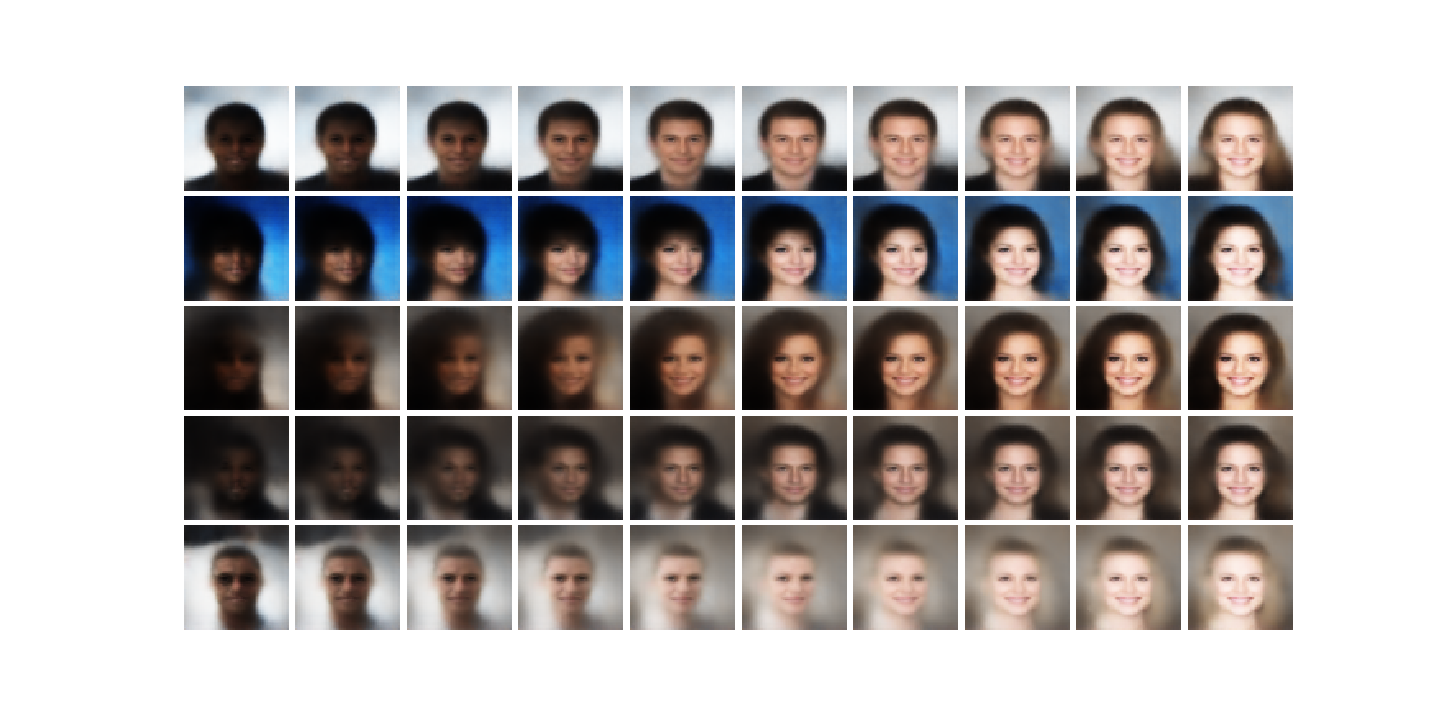}
          \caption{Skin Color}
      \end{subfigure}
     \begin{subfigure}[b]{0.329\linewidth}
          \centering
          \includegraphics[width=\linewidth,trim={6cm 3cm 5cm 3cm},clip]{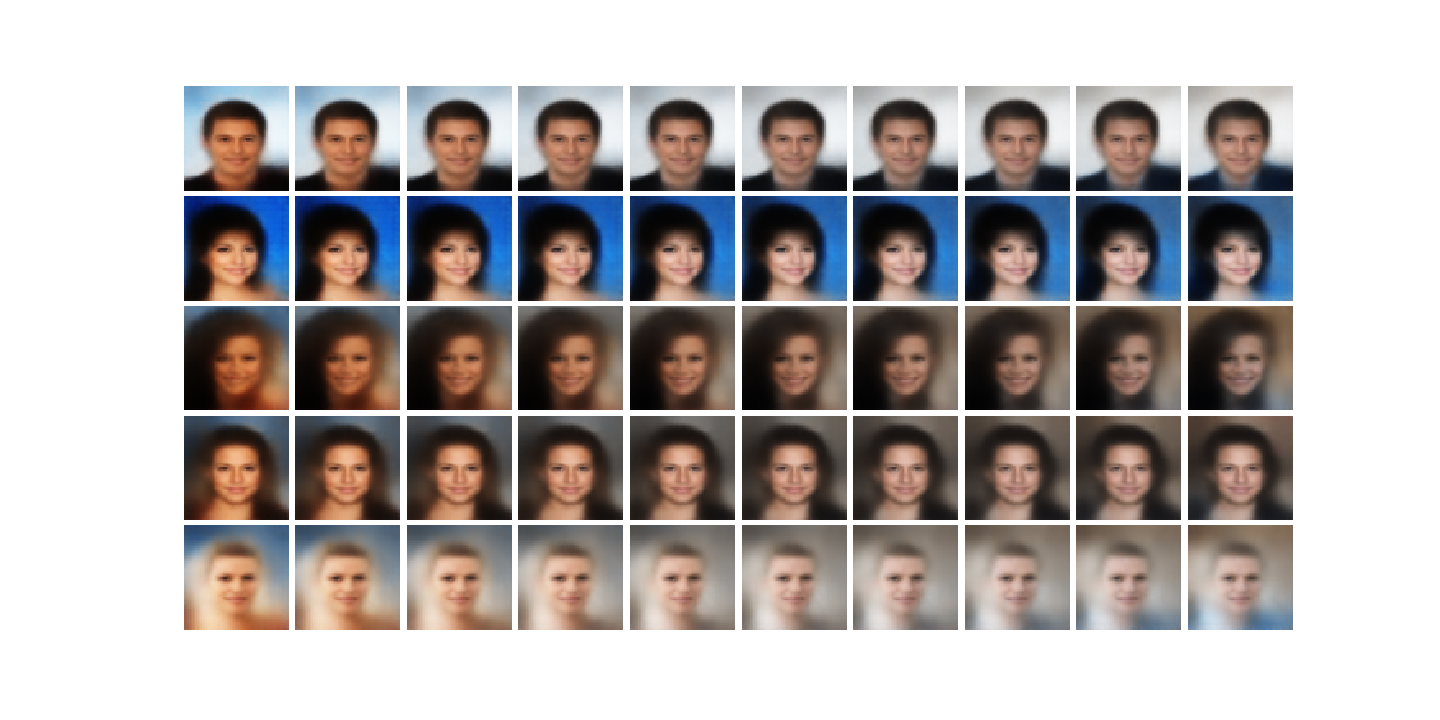}
          \caption{Image Saturation}
      \end{subfigure}
     \begin{subfigure}[b]{0.329\linewidth}
          \centering
          \includegraphics[width=\linewidth,trim={6cm 3cm 5cm 3cm},clip]{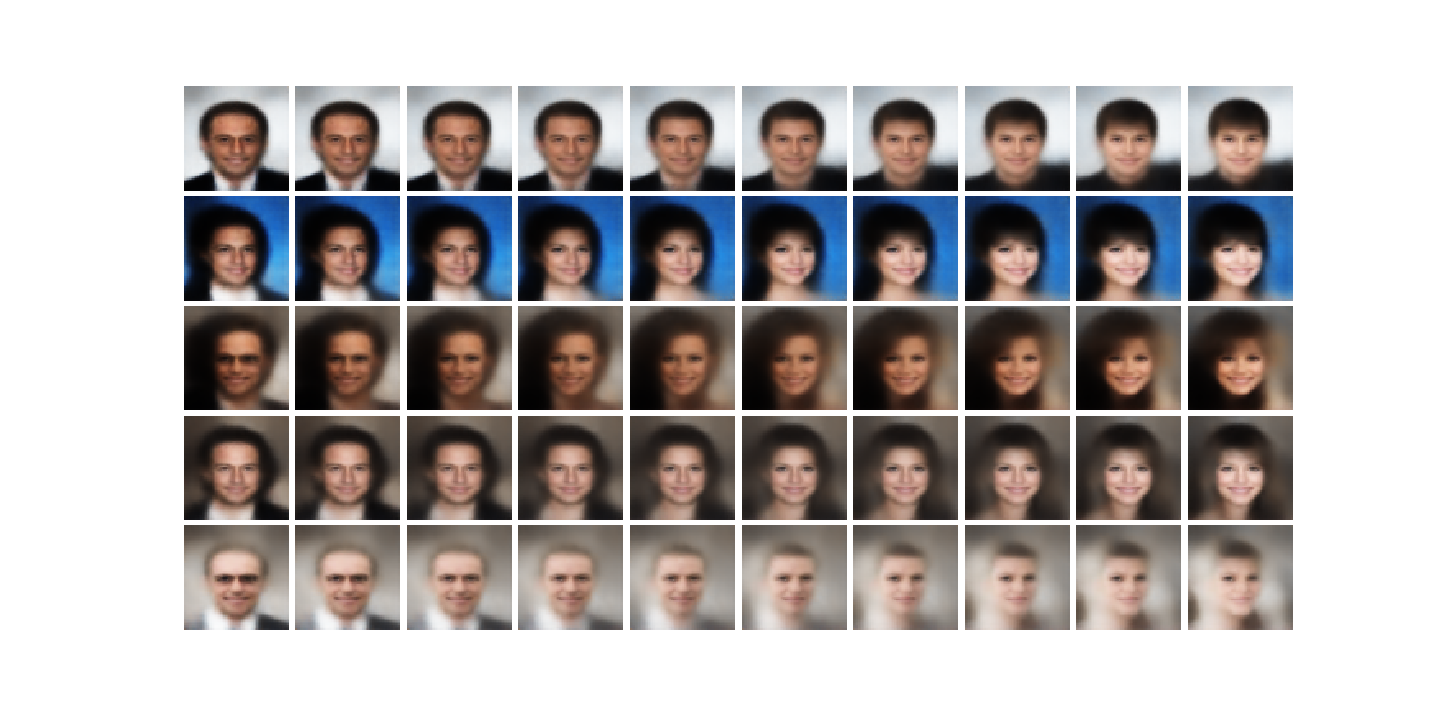}
          \caption{Gender}
      \end{subfigure}
     \begin{subfigure}[b]{0.329\linewidth}
          \centering
          \includegraphics[width=\linewidth,trim={6cm 3cm 5cm 3cm},clip]{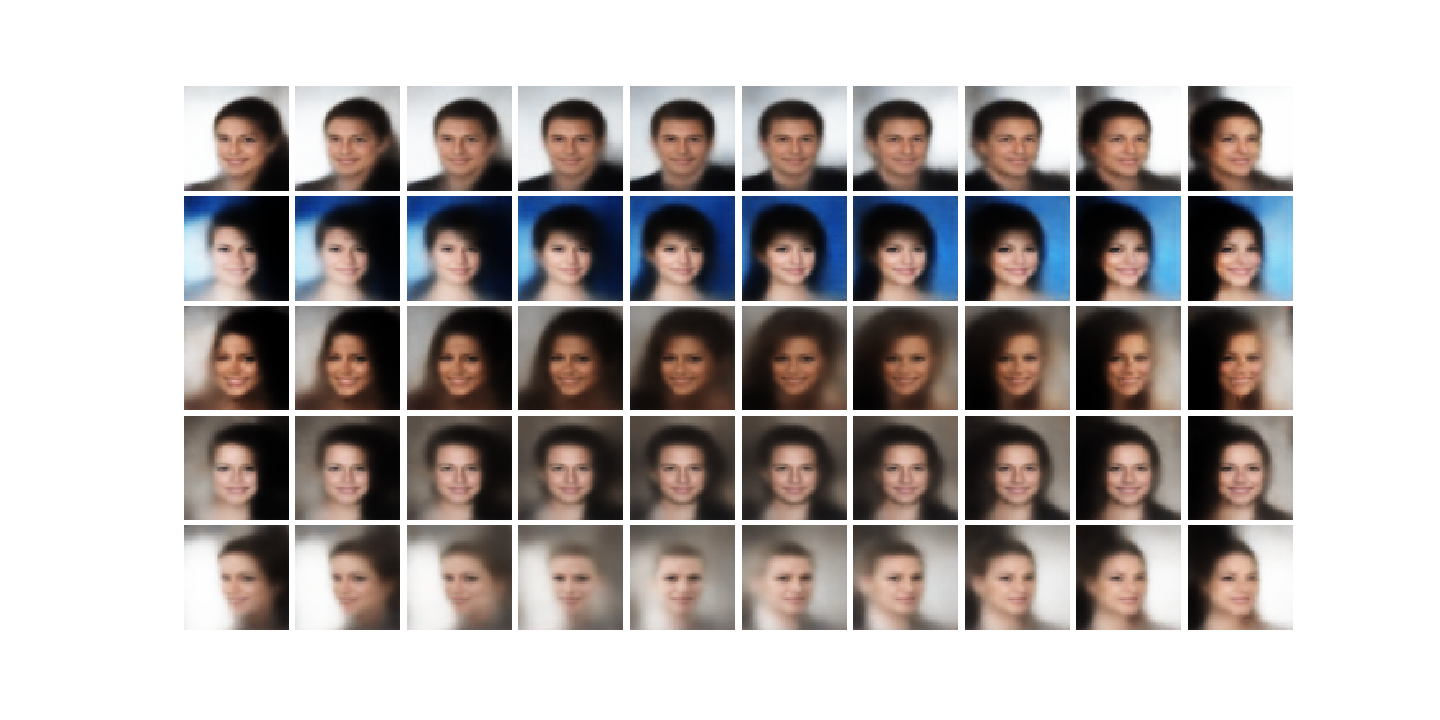}
          \caption{Rotation}
      \end{subfigure}
     \begin{subfigure}[b]{0.329\linewidth}
          \centering
          \includegraphics[width=\linewidth,trim={6cm 3cm 5cm 3cm},clip]{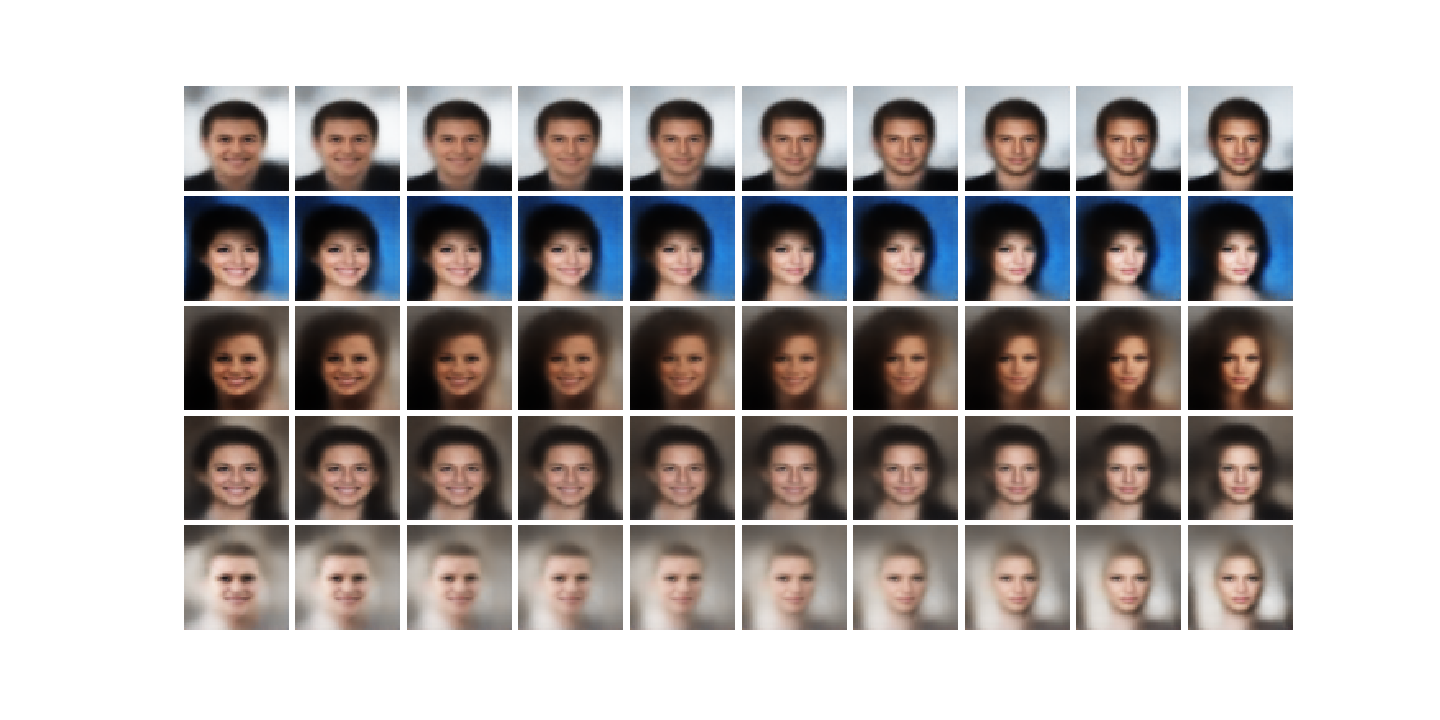}
          \caption{Emotion(Smile-Neutral)}
      \end{subfigure}
     \begin{subfigure}[b]{0.329\linewidth}
          \centering
          \includegraphics[width=\linewidth,trim={6cm 3cm 5cm 3cm},clip]{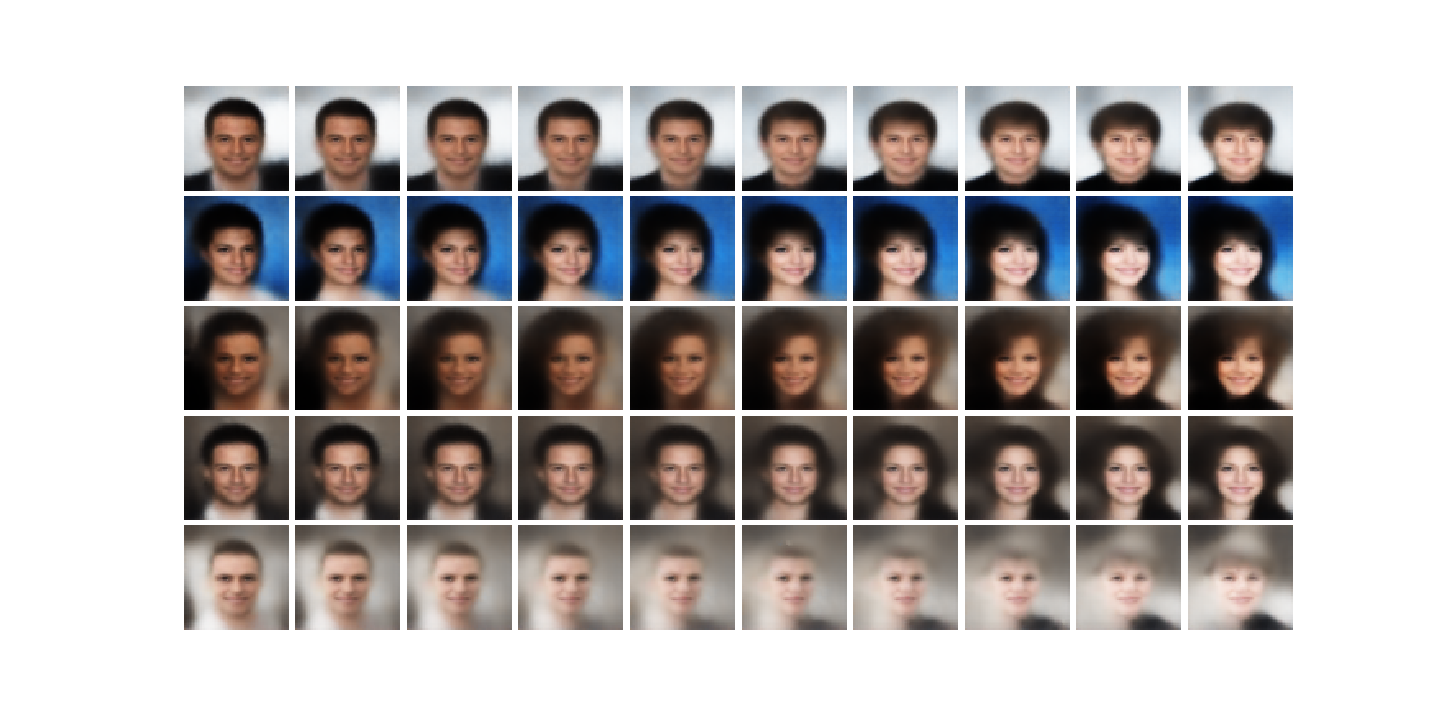}
          \caption{Hair(Fringe)}
      \end{subfigure} 
     \begin{subfigure}[b]{0.329\linewidth}
          \centering
          \includegraphics[width=\linewidth,trim={6cm 3cm 5cm 3cm},clip]{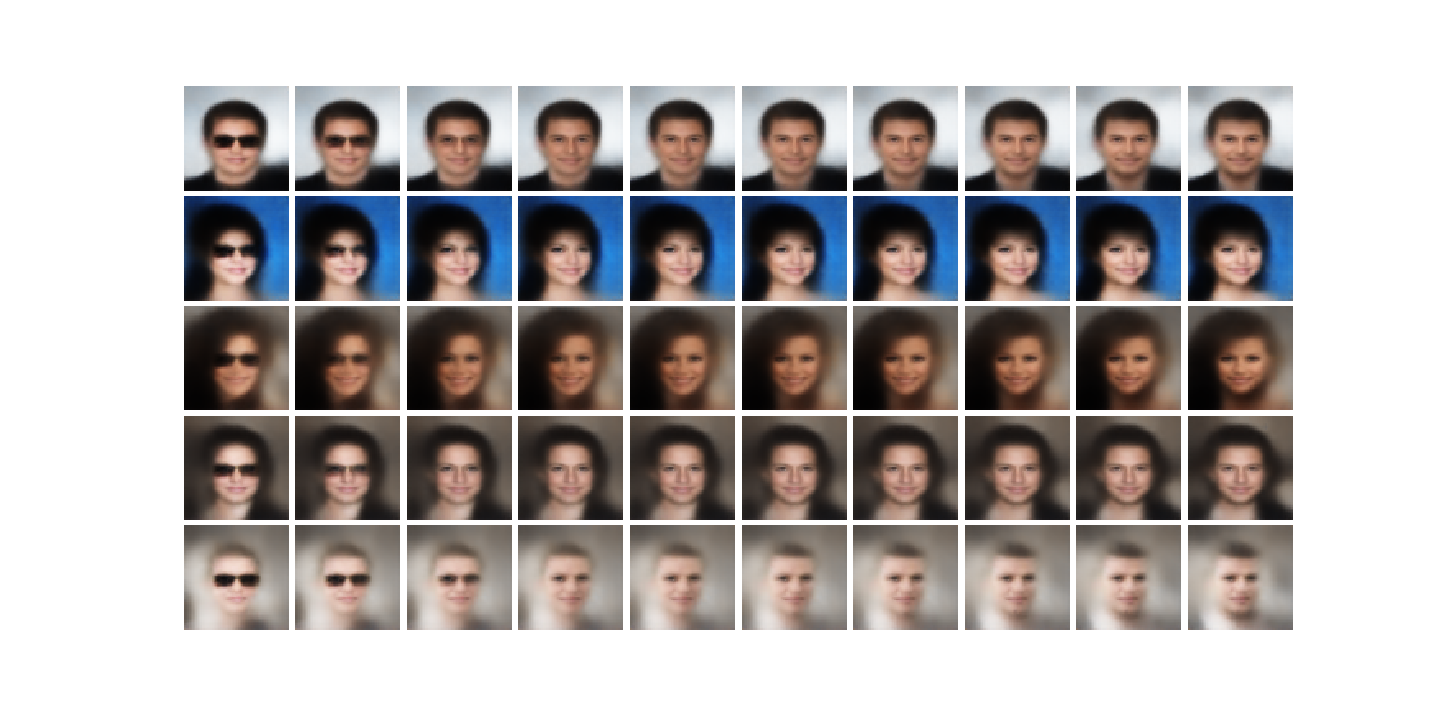}
          \caption{Sunglasses}
      \end{subfigure}  
     \begin{subfigure}[b]{0.329\linewidth}
          \centering
          \includegraphics[width=\linewidth,trim={6cm 3cm 5cm 3cm},clip]{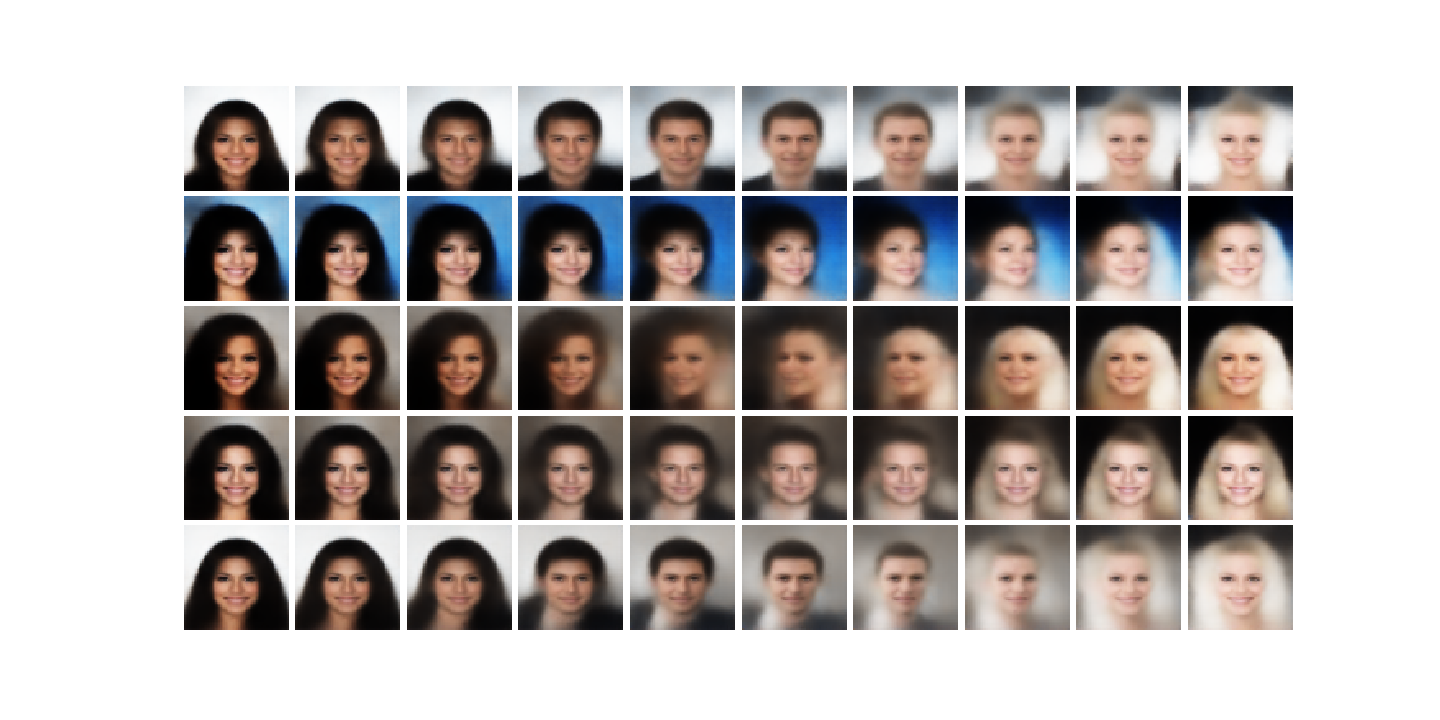}
          \caption{Hair Color(Black-White)}
      \end{subfigure}
     \begin{subfigure}[b]{0.329\linewidth}
          \centering
          \includegraphics[width=\linewidth,trim={6cm 3cm 5cm 3cm},clip]{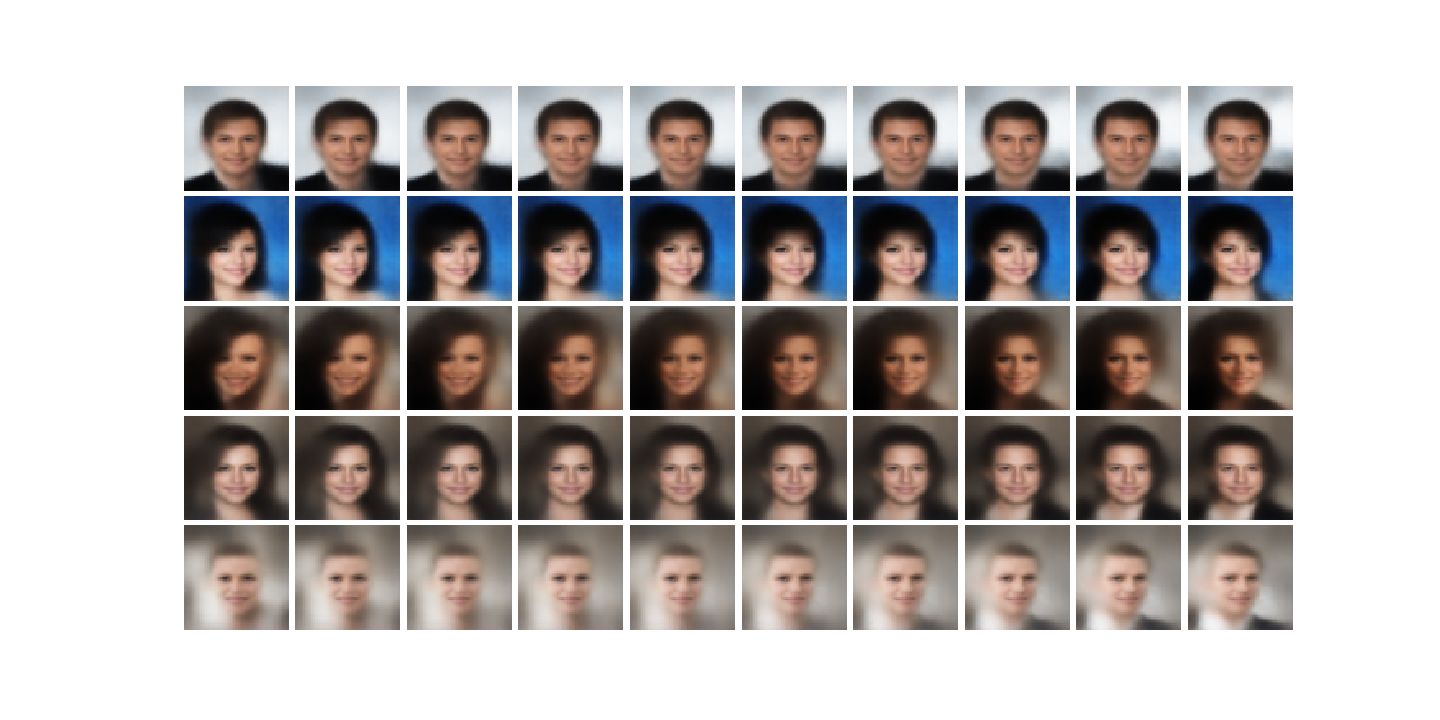}
          \caption{Hair (Side Part)}
          \label{fig:sidepart}
          \end{subfigure} 
   \begin{subfigure}[b]{0.329\linewidth}
          \centering
          \includegraphics[width=\linewidth,trim={6cm 3cm 5cm 3cm},clip]{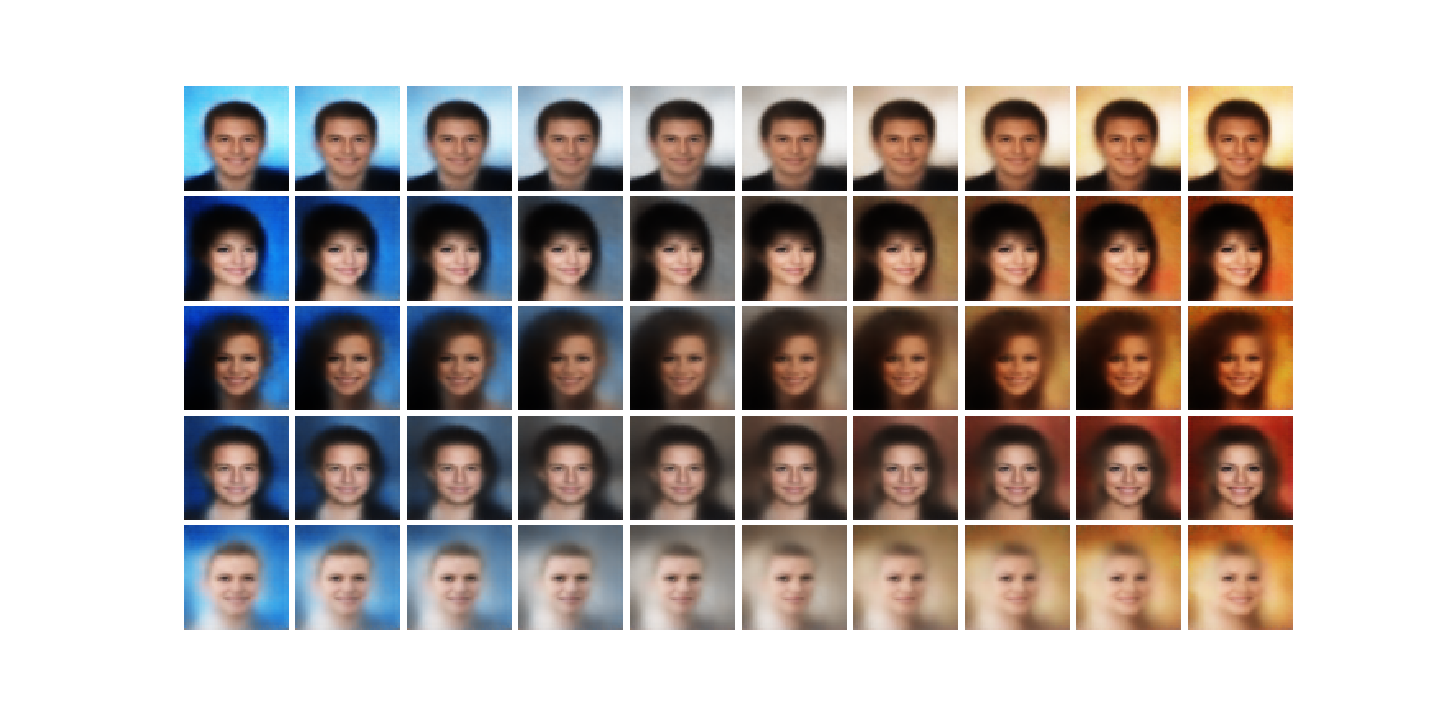}
          \caption{Background Tint (Red-Blue)}
          \label{fig:backcolor}
      \end{subfigure}
          \begin{subfigure}[b]{0.329\linewidth}
          \centering
          \includegraphics[width=\linewidth,trim={6cm 3cm 5cm 3cm},clip]{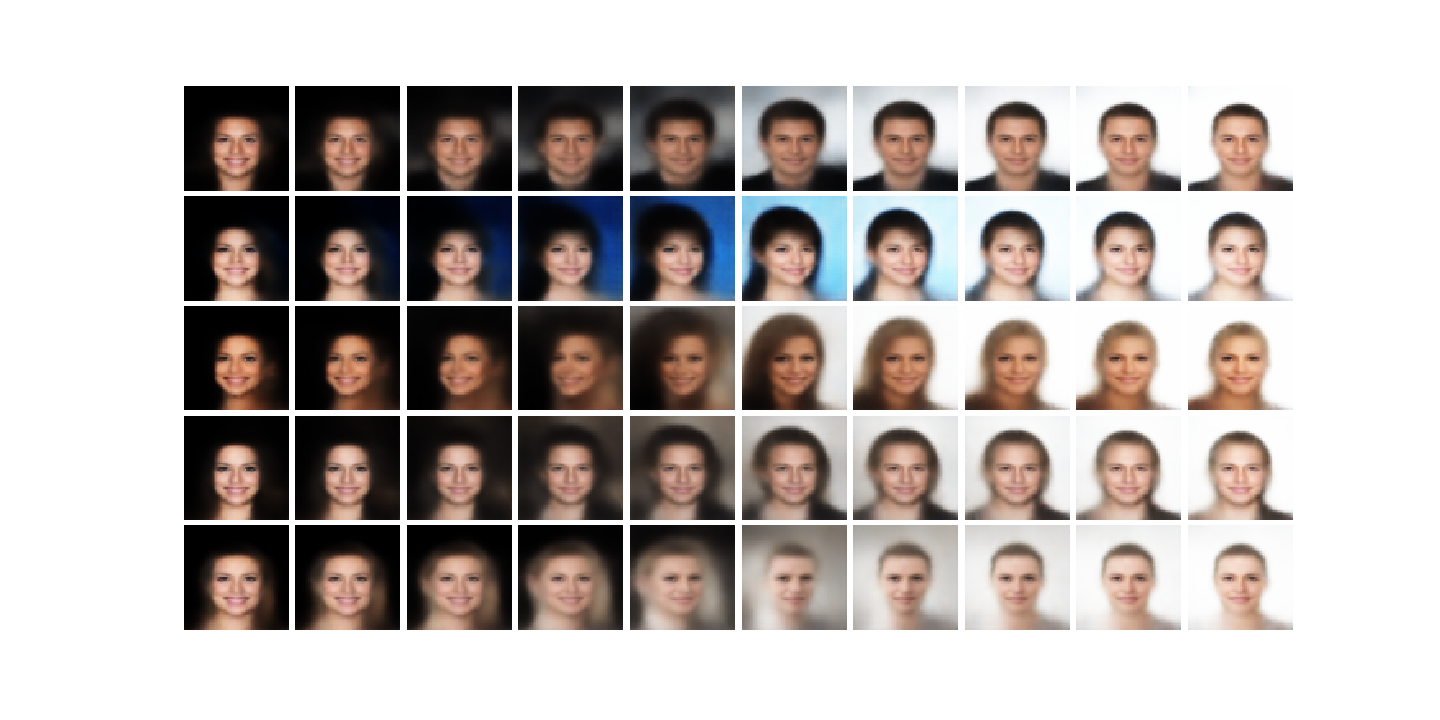}
          \caption{Background Brightness}
          \label{fig:brightness}
      \end{subfigure}
         \caption{ Disentangled factors found by Proj-VAE in CelebA dataset.
         \label{fig:celeba} }
 \end{figure}

\section{Discussion}

We demonstrate a simple projection in Proj-VAE can lead to surprisingly good
disentangling performance, bypassing the need of heuristic tuning. This
idea has an interesting support from the classic statistical theory. For concise exposition, we leave out several interesting extensions --- for example, what is the ideal number of latent dimensions $d$? This problem has recently been considered by \cite{kim2019relevance} via a prior shrinkage on $d$. In this article, we choose a conservatively small $d$, such that most of the factors can be interpretable. On the other hand, we did experiment with much larger $d$ (e.g. $d=5$ in Frey Face data, $d=32$ in chairs data) and the model seems to further capture variations in a smaller subset of data. To accommodate this, it appears intuitive to consider a mixture of Gaussians with an alternative projection scheme to remove correlation, for example, the Factor-VAE\citep{kim2018disentangling} takes the marginal distribution $q(z)$ into account which involves mixtures with a large number of components rather than the Gaussian, leading the model to learn more informative disentangled latent representations; 
Furthermore,
it remains to see if the disentanglement theory holds for mixture, and how to derive a solution as simple as Gaussian Proj-VAE.


\medskip
\small
 \bibliography{references.bib}

@inproceedings{higgins2017beta,
	title={$\beta$-{VAE}: learning basic visual concepts with a constrained variational framework},
	author={Higgins, Irina and Matthey, Loic and Pal, Arka and Burgess, Christopher and Glorot, Xavier and Botvinick, Matthew and Mohamed, Shakir and Lerchner, Alexander},
	booktitle={International Conference on Learning Representations},
	volume={3},
	year={2017}
}

@article{owen1983class,
	title={On the class of elliptical distributions and their applications to the theory of portfolio choice},
	author={Owen, Joel and Rabinovitch, Ramon},
	journal={The Journal of Finance},
	volume={38},
	number={3},
	pages={745--752},
	year={1983},
	publisher={Wiley Online Library}
}

@article{kim2019relevance,
	title={Relevance Factor {VAE}: Learning and Identifying Disentangled Factors},
	author={Kim, Minyoung and Wang, Yuting and Sahu, Pritish and Pavlovic, Vladimir},
	journal={arXiv preprint arXiv:1902.01568},
	year={2019}
}

@article{achille2018emergence,
	title={Emergence of invariance and disentanglement in deep representations},
	author={Achille, Alessandro and Soatto, Stefano},
	journal={The Journal of Machine Learning Research},
	volume={19},
	number={1},
	pages={1947--1980},
	year={2018},
	publisher={JMLR. org}
}

@article{goldstein1997stein,
	title={Stein's method and the zero bias transformation with application to simple random sampling},
	author={Goldstein, Larry and Reinert, Gesine},
	journal={The Annals of Applied Probability},
	volume={7},
	number={4},
	pages={935--952},
	year={1997},
	publisher={Institute of Mathematical Statistics}
}

@article{hjelm2018learning,
	title={Learning deep representations by mutual information estimation and maximization},
	author={Hjelm, R Devon and Fedorov, Alex and Lavoie-Marchildon, Samuel and Grewal, Karan and Trischler, Adam and Bengio, Yoshua},
	journal={arXiv preprint arXiv:1808.06670},
	year={2018}
}

@article{gotze2010rate,
	title={The rate of convergence of spectra of sample covariance matrices},
	author={G{\"o}tze, Friedrich and Tikhomirov, Aleksandr Nikolaevich},
	journal={Theory of Probability \& Its Applications},
	volume={54},
	number={1},
	pages={129--140},
	year={2010},
	publisher={SIAM}
}

@article{kingma2013auto,
	title={Auto-encoding variational {Bayes}},
	author={Kingma, Diederik P and Welling, Max},
	journal={arXiv preprint arXiv:1312.6114},
	year={2013}
}

@inproceedings{aubry2014seeing,
	title={Seeing 3d chairs: exemplar part-based 2d-3d alignment using a large dataset of cad models},
	author={Aubry, Mathieu and Maturana, Daniel and Efros, Alexei A and Russell, Bryan C and Sivic, Josef},
	booktitle={Proceedings of the IEEE conference on Computer Vision and Pattern Recognition},
	pages={3762--3769},
	year={2014}
}

@article{kim2018disentangling,
	title={Disentangling by factorising},
	author={Kim, Hyunjik and Mnih, Andriy},
	journal={arXiv preprint arXiv:1802.05983},
	year={2018}
}

@inproceedings{chen2016infogan,
	title={{InfoGAN}: Interpretable representation learning by information maximizing generative adversarial nets},
	author={Chen, Xi and Duan, Yan and Houthooft, Rein and Schulman, John and Sutskever, Ilya and Abbeel, Pieter},
	booktitle={Advances in Neural Information Processing Systems},
	pages={2172--2180},
	year={2016}
}

@article{stein1981estimation,
	title={Estimation of the mean of a multivariate normal distribution},
	author={Stein, Charles M},
	journal={Annals of Statistics},
	pages={1135--1151},
	year={1981},
	publisher={JSTOR}
}

@article{liu1994siegel,
	title={Siegel's formula via {Stein's} identities},
	author={Liu, Jun S},
	journal={Statistics \& Probability Letters},
	volume={21},
	number={3},
	pages={247--251},
	year={1994},
	publisher={Elsevier}
}

@inproceedings{liu2015deep,
	title={Deep learning face attributes in the wild},
	author={Liu, Ziwei and Luo, Ping and Wang, Xiaogang and Tang, Xiaoou},
	booktitle={Proceedings of the IEEE International Conference on Computer Vision},
	pages={3730--3738},
	year={2015}
}

@article{cogswell2015reducing,
	title={Reducing overfitting in deep networks by decorrelating representations},
	author={Cogswell, Michael and Ahmed, Faruk and Girshick, Ross and Zitnick, Larry and Batra, Dhruv},
	journal={arXiv preprint arXiv:1511.06068},
	year={2015}
}

@article{cheung2014discovering,
	title={Discovering hidden factors of variation in deep networks},
	author={Cheung, Brian and Livezey, Jesse A and Bansal, Arjun K and Olshausen, Bruno A},
	journal={arXiv preprint arXiv:1412.6583},
	year={2014}
}

@article{kumar2017variational,
	title={Variational inference of disentangled latent concepts from unlabeled observations},
	author={Kumar, Abhishek and Sattigeri, Prasanna and Balakrishnan, Avinash},
	journal={arXiv preprint arXiv:1711.00848},
	year={2017}
}
\end{document}